\documentclass[letterpaper,10pt,conference]{ieeeconf}
\IEEEoverridecommandlockouts       
\usepackage{amssymb}
\usepackage{amsmath}

\usepackage{amsthm}
\usepackage[english]{babel}
\usepackage{float}
\usepackage{graphicx}
\usepackage{hyperref}
\usepackage{mathtools}
\usepackage{booktabs}
\usepackage{afterpage}
\usepackage{dblfloatfix}
\usepackage{colortbl}

\usepackage{filecontents}
\usepackage[noadjust]{cite}
\usepackage[font=footnotesize,labelfont=footnotesize]{subcaption}
\usepackage[font=footnotesize,labelfont=footnotesize]{caption}
\usepackage{comment}
\usepackage{authblk}
\usepackage{multirow}
\usepackage{xcolor}
\usepackage{algorithm}
\usepackage{algorithmic}
\usepackage{xspace}
\usepackage{graphicx}
\usepackage[english]{babel}

\definecolor{darkblue}{rgb}{0.15,0.15,0.55}
\definecolor{lightgrey}{rgb}{0.75,0.75,0.75}

\newtheorem{theorem}{Theorem}

\newtheorem{lemma}[theorem]{Lemma}
\theoremstyle{definition}
\newtheorem{definition}{Definition}
\newtheorem{remark}{Remark}
\newtheorem{proposition}{Proposition}

\begin{document}
\title{\LARGE \bf Capability-Aware Heterogeneous Control Barrier Functions for Decentralized Multi-Robot Safe Navigation\vspace{-12pt}}

\author{Joonkyung Kim$^1$, Yanze Zhang$^2$, Wenhao Luo$^1$, Yiwei Lyu$^{1,*}$
\thanks{$^{1}$Department of Computer Science and Engineering, Texas A\&M University, $^{2}$Department of Computer Science, University of Illinois Chicago, $^*$Corresponding author: {\tt\small yiweilyu@tamu.edu}}
}


\maketitle 

\begin{abstract}

Safe decentralized navigation in heterogeneous robot teams is challenging because robots with different dynamics and control limits cannot satisfy identical collision-avoidance obligations equally well. Existing methods often rely on model-specific safety formulations or allocate avoidance responsibility without considering whether the assigned control is physically realizable,
{possibly leading to safety-constrained QP infeasibility and safety violations.}
We propose Capability-Aware Heterogeneous Control Barrier Functions (CA-HCBF),  a decentralized framework for consistent safety enforcement and capability-aware responsibility allocation. We derive a unified operational-space second-order control-affine formulation for holonomic and nonholonomic mobile robots through canonical transformation and backstepping, addressing the relative-degree mismatch while preserving forward invariance of the pairwise safe set. Second, we quantify each robot’s achievable control authority along the safety-critical direction using a support-function-based capability metric and use it to distribute pairwise avoidance responsibility. 
A feasibility-aware clipping mechanism further mitigates physically unrealizable responsibility assignments. Simulations with up to 30 heterogeneous robots and hardware experiments
demonstrate high task completion with low safety violations and reduction in local QP infeasibility.
\vspace{-0.5em}


\end{abstract}

\section{Introduction}
\label{sec:intro}
\vspace{-0.25em}

With recent advances in robotics, multi-robot systems (MRS) have gained
significant attention for collaborative tasks such as exploration and
mapping~\cite{gao2020random}, search-and-rescue~\cite{queralta2020search_rescue},
and warehouse automation~\cite{li2021lifelong}. However, as robot density
increases, safe navigation becomes increasingly challenging: frequent
interactions increase collision risk, while overly conservative avoidance
can significantly degrade team efficiency.

\begin{figure}[!t]
    \captionsetup{skip=0pt}
    \centering
    \begin{subfigure}{\linewidth}
        \centering
        \includegraphics[width=0.98\textwidth]{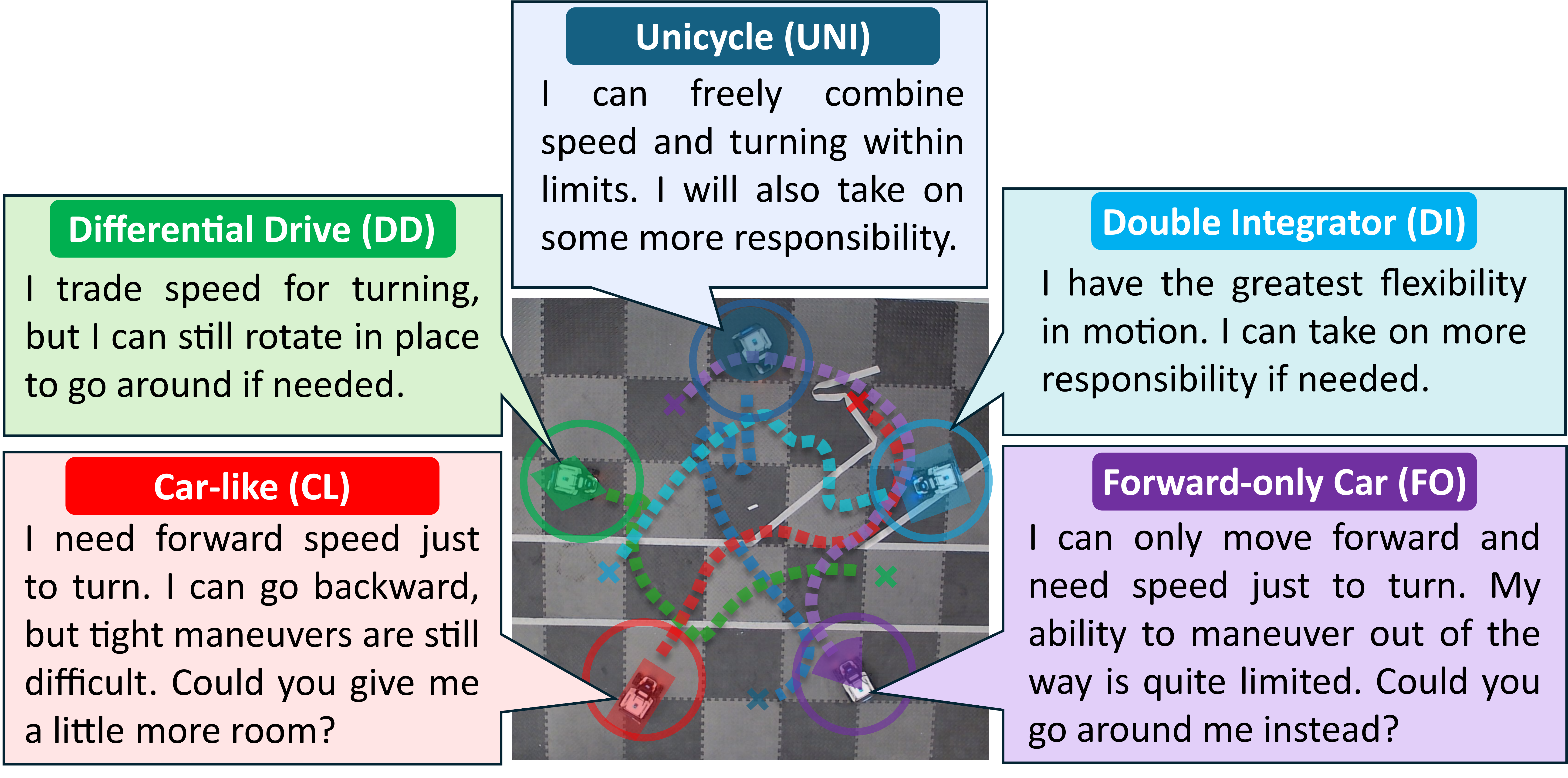}
        \caption{}
        \label{fig:dialogue}
    \end{subfigure}
    \vspace{0pt} 
    \begin{subfigure}{\linewidth}
        \centering
        \includegraphics[width=0.99\textwidth]{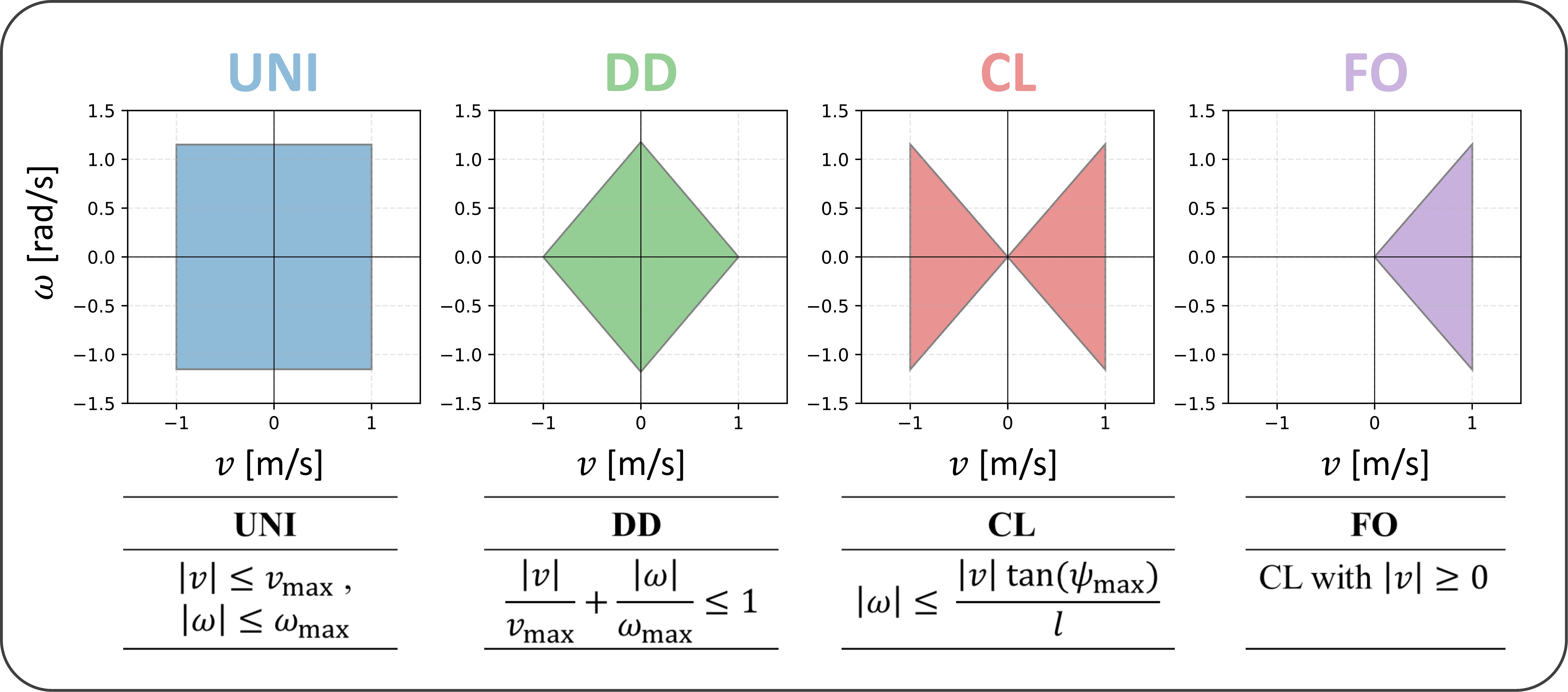}
        \caption{}
        \label{fig:canonical_viz} \vspace{2pt}
    \end{subfigure}
    
    \caption{
    \textbf{(a)} Five mobile robots with heterogeneous kinematic classes (DI, UNI, DD, CL, and FO) navigate toward individual goals. Their statements illustrate differences in directional motion capability, motivating the capability-aware responsibility allocation in Section~\ref{sec:method}. \textbf{(b)} Feasible velocity sets $\mathcal V_i$ in the $(v,\omega)$ space and their analytical definitions for the four nonholonomic kinematic classes~\cite{lynch2017modern_robotics} (Section~\ref{subsec:canonical}).}
    \label{fig:overview} \vspace{-2em}
\end{figure}

{
Control Barrier Functions (CBFs)~\cite{ames2019control} provide a principled
framework for safety enforcement by minimally modifying nominal control
commands through a quadratic program (QP) to maintain forward invariance of
a prescribed safe set~\cite{ames2016control_qp}. CBF-based multi-robot
navigation can be formulated in either centralized or decentralized form.
While centralized approaches coordinate all robot controls jointly, their
computational cost grows with team size~\cite{wang2017safety}. Decentralized
approaches improve scalability by allowing each robot to compute its own
control, but require the pairwise collision-avoidance responsibility to be
distributed among interacting robots~\cite{zeng2021safety,tan2021distributed}.
}

Despite this progress, decentralized safety enforcement for heterogeneous
mobile-robot teams remains challenging. Existing formulations often assume
homogeneous robot kinematics, under which pairwise safety responsibility can
be distributed uniformly among robots~\cite{wang2017safety,dawson2023safe}.
In heterogeneous teams, however, robots may differ substantially in their
available motion directions and realizable control authority, as illustrated
in Fig.~\ref{fig:overview}(a). Consequently, an equal responsibility split can
assign excessive avoidance effort to a kinematically constrained robot while
underutilizing a more maneuverable one. Such mismatched assignments can make
an otherwise reasonable pairwise safety obligation difficult or impossible
for the assigned robot to realize, thereby increasing the likelihood of local
safety-constrained QP infeasibility.

{These observations motivate two fundamental requirements for decentralized heterogeneous navigation:}
(i) a unified safety formulation that can be applied uniformly across heterogeneous kinematic models, and
(ii) a responsibility allocation mechanism that can truthfully assign avoidance effort according to each robot's capability while remaining physically realizable.
{To meet these requirements, we propose the \textit{Capability-Aware Heterogeneous Control Barrier Function} (\textit{CA-HCBF}), which unifies heterogeneous robot dynamics under a 
{unified} safety formulation and allocates pairwise avoidance responsibilities based on directional motion capability.}
Our contributions are threefold:
{(1) A unified second-order CBF formulation that enables consistent decentralized safety enforcement across five heterogeneous holonomic and nonholonomic mobile robot models;}
{(2) A support-function-based capability-aware responsibility allocation framework with feasibility-aware clipping that distributes pairwise avoidance effort according to each robot's maneuverability;}
and {(3) Extensive simulations with up to 30 robots and five-robot hardware experiments demonstrating higher task completion, fewer safety violations, and improved decentralized feasibility over representative baselines.}
\vspace{-1.5em}

\section{Related Work}
\label{sec:relWork}
\vspace{-0.2em}

Multi-agent collision avoidance has been studied through several classes of approaches reflecting different tradeoffs among safety guarantees, scalability, and cooperation. 
Velocity obstacle (VO)-based methods~\cite{fiorini1998vo, van2011orca} are computationally efficient but provide guarantees only in velocity space under finite-horizon assumptions. Artificial Potential Field (APF)-based methods~\cite{khatib1986apf, zhang2023rpf} offer lightweight reactive avoidance but lack formal continuous-time safety guarantees and may suffer from local minima. Model Predictive Control (MPC)-based approaches~\cite{dai2017distributed_mpc, luis2020online_mpc} enforce safety over a receding horizon but scale poorly with the number of agents. {In contrast,} CBF-based methods~\cite{ames2019control} {enforce forward invariance of a safe set by minimally modifying a nominal controller through a real-time quadratic program, and can be implemented in either centralized or decentralized form}~\cite{wang2017safety}. {Centralized CBF formulations jointly optimize all agents’ control inputs and therefore enable coordinated avoidance, but the number of pairwise constraints grows quadratically with team size. Decentralized formulations improve scalability by allowing each agent to solve a local QP, but coordination must then be recovered through how each pairwise safety constraint is divided between the interacting agents.}

{This need has motivated responsibility-allocation strategies that distribute pairwise collision-avoidance effort between agents}
~\cite{cosner2023learning_res, remy2025learning_res_alloc}. Existing methods, however, primarily consider parametric heterogeneity within a shared kinematic model (e.g., differences in robot size, speed limits, or dynamic parameters)~\cite{wang2016safety, lyu2022responsibility}. 
{Heterogeneous mobile robot teams introduce a more fundamental challenge: robots governed by different kinematic constraints and admissible control sets may not admit a common pairwise CBF formulation, and the same assigned responsibility may require substantially different, or physically unrealizable, control actions across robot types}~\cite{xiao2019hocbf}.
{One conservative alternative is to treat neighboring robots as worst-case dynamic obstacles~\cite{kim2025srcbf}, but doing so discards reciprocal cooperation and may lead to overly conservative behavior or deadlock in dense interactions}~\cite{zhang2025adaptive, lee2025merry}.
Therefore, responsibility allocation in heterogeneous mobile robot teams requires both a unified safety formulation across distinct kinematic models and a capability-aware mechanism that assigns avoidance effort according to each robot’s realizable maneuvering authority.


{Beyond pairwise constraint construction, decentralized CBF performance
also depends on how the required avoidance effort is distributed among
interacting robots. Responsibility assignments that exceed a robot's
directional control authority can shrink or eliminate the feasible region
of its local safety-constrained QP, particularly for kinematically
restricted robots~\cite{zeng2021pairwise_feasibility,tan2021distributed, wang2017safety}. Existing remedies, including constraint relaxation~\cite{ames2016control_qp}, priority hierarchies~\cite{lee2023hierarchical}, and safety-margin inflation~\cite{chen2020guaranteed}, recover feasibility by weakening safety requirements, prioritizing selected constraints, or increasing conservatism. }
{
{Our approach instead mitigates infeasibility upstream, at the responsibility-allocation stage, by shifting avoidance effort away from robots with limited directional control authority and toward those better able to realize the required separating motion. This capability-aware allocation effectively reduces the likelihood that physically unrealizable pairwise obligations propagate into infeasible local safety-constrained QPs.}
}
\vspace{-5pt}

\section{Preliminary} \label{sec:prelim}
\vspace{-3pt}

We briefly review the CBF framework for multi-robot collision avoidance and the canonical velocity-level kinematic representation for nonholonomic mobile robots, which forms the basis of the unified acceleration-level formulation introduced in Section~\ref{sec:method}. \vspace{-0.4em}

\subsection{Control Barrier Function (CBF)}
\label{ssec:cbf}

CBFs~\cite{ames2019control} provide a 
systematic approach to characterize admissible control inputs that guarantee safety for control-affine systems $\dot{\mathbf{x}} = f(\mathbf{x}) + g(\mathbf{x}) \mathbf{u}$, where $\mathbf{x} \in \mathcal{X} \subset \mathbb{R}^n$ is the system state, $\mathbf{u} \in \mathcal{U} \subset \mathbb{R}^m$ is the control input, $\mathcal{U}$ is the set of admissible inputs, and $f$, $g$  are locally Lipschitz. Given a continuously differentiable function $h : \mathbb{R}^n \to \mathbb{R}$, define the safe set $\mathcal{H} = \{ \mathbf{x} \in \mathcal{X} : h(\mathbf{x}) \ge 0 \}$.
The time derivative of $h$ along system trajectories is
$\dot{h}(\mathbf{x}, \mathbf{u}) = L_f h(\mathbf{x}) + L_g h(\mathbf{x}) \mathbf{u}$,
where $L_f h$ and $L_g h$ denote Lie derivatives along $f$ and $g$.


\begin{definition}[Control Barrier Function~\cite{ames2019control}]
A continuously differentiable function $h:\mathbb{R}^n\rightarrow\mathbb{R}$ is a CBF if there exists an extended class $\mathcal{K}$ function $\kappa(\cdot)$ such that
$\sup_{\mathbf u\in\mathcal U}\!\left(L_fh+L_gh\,\mathbf u\right)\ge -\kappa(h)$.
Adopting $\kappa(h) = \gamma h$ with user-defined parameter $\gamma \geq 0$~\cite{ames2019control, wang2017safety}, the admissible safe control set is
$\mathcal B(\mathbf x)=\{\mathbf u\in\mathcal U: L_fh+L_gh\,\mathbf u+\gamma h(\mathbf x)\ge0\}$.
Any Lipschitz controller satisfying $\mathbf u(\mathbf x)\in\mathcal B(\mathbf x)$ renders the safe set $\mathcal H=\{\mathbf x:h(\mathbf x)\ge0\}$ forward invariant~\cite{ames2019control}.
\end{definition}


This framework extends to multi-robot collision avoidance~\cite{wang2016safety}. For each pair $(i,j)$, we define the pairwise safety function
$h_{ij}(\mathbf{p}_i,\mathbf{p}_j):=\|\mathbf{p}_i-\mathbf{p}_j\|^2-R_{\mathrm{safe}}^2$,
where $R_{\mathrm{safe}}\in\mathbb{R}$ is the safety margin and
$\mathbf{p}_i,\mathbf{p}_j\in\mathbb{R}^2$ denote the positions of agents $i$ and $j$ extracted from their states $\mathbf{x}_i$ and $\mathbf{x}_j$.

We denote the safe set $\mathcal{C}_{ij} := \{(\mathbf{p}_i, \mathbf{p}_j) \mid h_{ij} \geq 0\}$. The admissible safe control set enforcing forward invariance of $\mathcal{C}_{ij}$ is \vspace{-0.2em}
\begin{equation}
\begin{aligned}
    \mathcal{B}_{ij}(\mathbf{x}) = \big\{ \mathbf{u} \in \mathcal{U} 
    :\, &L_f h_{ij} + L_g h_{ij}\,\mathbf{u} 
    + \gamma_{ij}\,h_{ij} \geq 0 \big\},
\end{aligned} \label{eq:foward_control_set}
\end{equation}  
where $\gamma_{ij} \geq 0$ governs the response aggressiveness near the safety boundary. Any controller satisfying $\mathbf{u} \in 
\mathcal{B}_{ij}(\mathbf{x})$ for all pairs $(i,j)$ guarantees 
collision-free operation~\cite{ames2019control, wang2017safety}.

\subsection{Canonical Representation for Nonholonomic Agents}
\label{subsec:canonical}

Unlike holonomic agents, nonholonomic wheeled robots are subject to kinematic constraints that prohibit instantaneous lateral motion, and thus impose various constraints in terms of achievable $\mathbf{u}$ when designing effective individual robot controller to collaboratively enforce Eq.~\ref{eq:foward_control_set} for safety. 
Although these platforms differ in actuation mechanisms and physical configurations, they admit a common simplified kinematic representation under the assumption that steering dynamics are handled by a lower-level controller~\cite{lynch2017modern_robotics}.


\begin{definition}[Canonical Simplified Model~\cite{lynch2017modern_robotics}]
\label{def:canonical_kine}
The configuration of agent $i$ is
$\mathbf q_i=[\phi_i,x_i,y_i]^\top$, where
$(x_i,y_i)\in\mathbb R^2$ is the Cartesian position and
$\phi_i\in(-\pi,\pi]$ is the heading angle.
For the nonholonomic models, define the velocity state
$\boldsymbol{\nu}_i=[v_i,\omega_i]^\top$, where $v_i$ and $\omega_i$
denote the forward and angular velocities, respectively.
The kinematics are
\begin{equation}
    \dot{\mathbf q}_i
    =
    \begin{bmatrix}
        0 & 1 \\
        \cos\phi_i & 0 \\
        \sin\phi_i & 0
    \end{bmatrix}
    \boldsymbol{\nu}_i,
    \label{eq:canonical_kine}
\end{equation}
which encodes the nonholonomic constraint
$\dot y_i\cos\phi_i-\dot x_i\sin\phi_i=0$,
indicating that the agent can move instantaneously only along its
heading direction.
\end{definition}
\vspace{-0.5em}


The distinct kinematic characteristics of each agent type are captured
by the feasible velocity sets $\mathcal V_i$, parameterized by the
maximum forward speed $v_{\max}$, maximum angular velocity
$\omega_{\max}$, maximum steering angle $\psi_{\max}$, and wheelbase
$\ell$. Their geometry in the $(v,\omega)$ space is illustrated in
Fig.~\ref{fig:canonical_viz}. Thus, while all nonholonomic agents share the same kinematic structure in~\eqref{eq:canonical_kine}, their motion
constraints are encoded by the corresponding $\mathcal V_i$.


The canonical representation alone does not directly permit independent
control of the Cartesian position $(x_i,y_i)$ because of the
nonholonomic constraint. We therefore adopt the \emph{reference-point}
transformation~\cite{lynch2017modern_robotics} that maps the velocity
state to the Cartesian velocity of an offset point.

\begin{definition}[Reference-Point Operational State]
\label{def:refpoint}
For agent $i$, define a reference point at a look-ahead distance
$x_r>0$ along the heading direction, with Cartesian position
\vspace{-0.5em}
\begin{equation}
    \mathbf p_i(\mathbf q_i)
    =
    \begin{bmatrix}
        x_i+x_r\cos\phi_i \\
        y_i+x_r\sin\phi_i
    \end{bmatrix}.
    \label{eq:refpoint}
\end{equation}
\end{definition}

\begin{lemma}[Reference-Point Velocity Mapping]
\label{lem:feedback_lin}
The time derivative of $\mathbf p_i$ satisfies
\begin{equation}
    \dot{\mathbf p}_i
    =
    J(\mathbf q_i)\boldsymbol{\nu}_i,
\end{equation}
where
\begin{equation}
    J(\mathbf q_i)
    =
    \begin{bmatrix}
        \cos\phi_i & -x_r\sin\phi_i \\
        \sin\phi_i &  x_r\cos\phi_i
    \end{bmatrix}.
    \label{eq:Jacobian}
\end{equation}
Since $\det(J(\mathbf q_i))=x_r>0$ for all $\phi_i$,
$J$ is invertible and
$\boldsymbol{\nu}_i\mapsto\dot{\mathbf p}_i$
is a bijective linear transformation.
\end{lemma}
\noindent For a detailed derivation of the reference-point transformation, we
refer the reader to~\cite{lynch2017modern_robotics}.

Since any desired operational-space velocity can be mapped to
$\boldsymbol{\nu}_i=J^{-1}(\mathbf q_i)\dot{\mathbf p}_i$,
the reference-point transformation provides a common Cartesian
velocity representation for the nonholonomic models.
For the holonomic Double Integrator, we use the physical position
directly, corresponding to $J=I_2$.
Thus, all five kinematic classes admit the common operational-space
representation developed in Section~\ref{sec:method}.
\vspace{-0.8em}

\section{Capability-Aware Decentralized CBF for Heterogeneous Multi-Robot Systems}
\label{sec:method}


{

This section presents a decentralized CBF framework for safe navigation in heterogeneous mobile robot teams with distinct kinematic constraints. 
To address these challenges, we propose the Capability-Aware Heterogeneous Control Barrier Function (CA-HCBF) framework. CA-HCBF first maps heterogeneous kinematic models into a common second-order operational-space representation and constructs pairwise backstepping CBF constraints in that space. It then allocates pairwise avoidance responsibilities according to each robot’s directional motion capability and applies feasibility-aware clipping to limit individually unrealizable assignments.
} \vspace{-0.4em}

\subsection{System Model}
\label{subsec:system_model} 

We consider a heterogeneous team of $N$ mobile robots, indexed by  $i \in \{1,\ldots,N\}$, operating in a shared 2D workspace. Each robot belongs to one of five kinematic classes: holonomic Double Integrator (DI), or nonholonomic Unicycle (UNI), Differential-Drive (DD), Car-Like (CL), and Forward-Only (FO). 
For the nonholonomic models, the configuration is
$\mathbf q_i=[\phi_i,x_i,y_i]^\top$ and the velocity state is
$\boldsymbol{\nu}_i=[v_i,\omega_i]^\top$, as defined in
Eq.~\eqref{eq:canonical_kine}, yielding the full state
$\mathbf x_i=[\mathbf q_i^\top,\boldsymbol{\nu}_i^\top]^\top$.
For DI, we use the Cartesian position and velocity directly, with
$\boldsymbol{\nu}_i=[v_{x,i},v_{y,i}]^\top$.
All robot classes are represented as second-order control-affine systems
    $\dot{\mathbf x}_i
    =
    f_i(\mathbf x_i)
    +
    g_i(\mathbf x_i)\mathbf u_i,$
where $\mathbf u_i=\dot{\boldsymbol{\nu}}_i\in\mathbb R^2$ denotes the generalized acceleration input. Specifically,
$\mathbf u_i=[\dot v_i,\dot\omega_i]^\top$ for the nonholonomic models and $\mathbf u_i=[a_{x,i},a_{y,i}]^\top$ for DI.

Let $\mathcal V_i$ denote the model-specific feasible velocity set defined in Section~\ref{subsec:canonical}, and let $\mathcal A_i$
denote the actuator acceleration limits. The state-dependent
acceleration-level admissible set used by the QP is
\begin{equation}
    \mathcal U_i^{\mathrm{acc}}(\mathbf x_i)
    :=
    \left\{
        \mathbf u_i\in\mathcal A_i
        \,\middle|\,
        \boldsymbol{\nu}_i+\Delta t\,\mathbf u_i\in\mathcal V_i
    \right\}.
    \label{eq:acc_admissible}
\end{equation}
Thus, $\mathcal V_i$ encodes the model-specific velocity constraints,
whereas $\mathcal U_i^{\mathrm{acc}}(\mathbf x_i)$ contains the
acceleration inputs that satisfy both the actuator limits and the
next-step velocity constraints.




\subsection{Unified Operational-Space Model}
\label{subsec:uni_op_space}
Direct CBF synthesis on $\mathbf{q}_i$ is complicated by nonholonomic 
constraints and the heterogeneous structure of $f_i, g_i$ across 
kinematic classes. We address both issues by expressing all agents in a shared operational space via the reference-point transformation (Definition~\ref{def:refpoint}, Lemma~\ref{lem:feedback_lin}), which yields a 2D output $\mathbf{p}_i$ with dynamics affine in $\mathbf{u}_i$ for all five kinematic classes, enabling a uniform CBF formulation regardless of the underlying kinematic structure.

Differentiating $\dot{\mathbf{p}}_i = J(\mathbf{q}_i)\boldsymbol{\nu}_i$ 
yields the second-order operational-space dynamics: \vspace{-0.0em}
\begin{equation}
    \ddot{\mathbf{p}}_i 
    = \underbrace{\dot{J}(\mathbf{q}_i,\boldsymbol{\nu}_i)\,\boldsymbol{\nu}_i}
      _{\boldsymbol{\eta}_i(\mathbf{x}_i)}
    + J(\mathbf{q}_i)\,\mathbf{u}_i,
    \label{eq:op_second}
\end{equation}
where the drift term $\boldsymbol{\eta}_i(\mathbf{x}_i) \in \mathbb{R}^2$ captures the centripetal acceleration induced by heading rotation, which is nonzero whenever the robot is turning ($\omega_i \neq 0$):
\begin{equation}
    \boldsymbol{\eta}_i 
    = \omega_i\begin{bmatrix}
        -\sin\phi_i & -x_r\cos\phi_i \\
         \cos\phi_i & -x_r\sin\phi_i
      \end{bmatrix}\boldsymbol{\nu}_i.
\end{equation}
This yields the unified control-affine form
\begin{equation}
    \ddot{\mathbf{p}}_i = \boldsymbol{\eta}_i(\mathbf{x}_i) 
    + G_i(\mathbf{x}_i)\,\mathbf{u}_i,
    \label{eq:unified}
\end{equation}
%
where $G_i:=J(\mathbf{q}_i)$, with the DI corresponding to the special case $G_i=I_2$ and
$\boldsymbol{\eta}_i=\mathbf{0}$.
By Lemma~\ref{lem:feedback_lin}, all five kinematic classes admit the unified operational-space representation
$(\mathbf{p}_i,\boldsymbol{\eta}_i,G_i)$.

\vspace{-0.2em}
\subsection{Pairwise Safety via Backstepping CBF}
\label{subsec:backstepping}
\subsubsection{Relative Degree and Backstepping Construction}

{To construct pairwise CBF constraints under acceleration-level control, we apply a backstepping construction~\cite{taylor2022safe_backstep} to the pairwise safety function.}
For every pair $(i,j)$ with $i \neq j$, let $\Delta\mathbf{p}_{ij}:= \mathbf{p}_i - \mathbf{p}_j$ and $\Delta\dot{\mathbf{p}}_{ij}:= \dot{\mathbf{p}}_i - \dot{\mathbf{p}}_j$. We define the pairwise safety function $h_{ij} := \|\Delta\mathbf{p}_{ij}\|^2 - R_{ij}^2$, where $R_{ij} > 0$ is the predefined safety distance (see 
Section~\ref{subsec:safety_guarantee}), with safe set $\mathcal{C} := \{\mathbf{x} \mid h_{ij}(\mathbf{x}) \geq 0,\ \forall\, i \neq j\}$. The first two time derivatives are
\begin{equation}
    \dot{h}_{ij} = 2\,\Delta\mathbf{p}_{ij}^{\top}\Delta\dot{\mathbf{p}}_{ij},
    \ \
    \ddot{h}_{ij} = 2\|\Delta\dot{\mathbf{p}}_{ij}\|^2 
    + 2\,\Delta\mathbf{p}_{ij}^{\top}\Delta\ddot{\mathbf{p}}_{ij}.
    \label{eq:hij_derivatives}
\end{equation} 
Since $\mathbf{u}_i$ enters only through $\ddot{\mathbf{p}}_i$ in 
$\ddot{h}_{ij}$, $h_{ij}$ has relative degree two with respect to 
$\mathbf{u}_i$, precluding direct application of a standard first-order CBF~\cite{xiao2019hocbf}. To address this, we employ a backstepping construction~\cite{krstic2006backstepping, taylor2022safe_backstep} that augments $h_{ij}$ with a velocity-dependent term. For design parameters $\lambda_1, \lambda_2 > 0$, define $\psi_{ij} := \dot{h}_{ij} + \lambda_1\, h_{ij}$
and the augmented safe set
\begin{equation}
    \mathcal{C}_{ij}^{\mathrm{aug}} := \big\{ (\mathbf{x}_i, \mathbf{x}_j) 
    \mid h_{ij} \geq 0 \;\wedge\; \psi_{ij} \geq 0 \big\}.
\end{equation}
Intuitively, $\psi_{ij} \geq 0$ acts as a velocity-level safety margin: as the pair approaches the boundary ($h_{ij} \to 0$), the 
allowable closing rate $\dot{h}_{ij}$ must also approach zero, 
preventing agents from reaching the boundary at high relative speed. 
Applying a first-order CBF condition to $\psi_{ij}$ via $\dot{\psi}_{ij} + \lambda_2\psi_{ij} \geq 0$ and expanding $\dot{\psi}_{ij} = \ddot{h}_{ij} + \lambda_1\dot{h}_{ij}$ yields
\begin{equation}
    \ddot{h}_{ij} + (\lambda_1 + \lambda_2)\dot{h}_{ij} 
    + \lambda_1\lambda_2\,h_{ij} \geq 0.
    \label{eq:second_order_cbf}
\end{equation}

\subsubsection{Coupled Constraint and Decentralized Decomposition}

{To enable decentralized implementation, we derive the coupled pairwise safety constraint and then decompose it into per-agent local constraints.}
Substituting~\eqref{eq:unified} and~\eqref{eq:hij_derivatives} into~\eqref{eq:second_order_cbf}, the joint constraint coupling both agents is:
\begin{equation}
    2\,\Delta\mathbf{p}_{ij}^{\top} 
    \big(G_i\mathbf{u}_i - G_j\mathbf{u}_j\big)
    + 2\,\Delta\mathbf{p}_{ij}^{\top}(\boldsymbol{\eta}_i - \boldsymbol{\eta}_j) 
    + \Upsilon_{ij} \geq 0,
    \label{eq:global_constraint}
\end{equation}
where $\Upsilon_{ij} := 2\|\Delta\dot{\mathbf{p}}_{ij}\|^2 + 
(\lambda_1+\lambda_2)\dot{h}_{ij} + \lambda_1\lambda_2\,h_{ij}$ 
collects all state-dependent terms, with $\Upsilon_{ij} = \Upsilon_{ji}$. 
Since~\eqref{eq:global_constraint} couples $\mathbf{u}_i$ and 
$\mathbf{u}_j$, solving it jointly would require centralized 
coordination. To enable decentralized execution, we introduce a 
responsibility parameter $\alpha_{ij} \in [0, 1]$ that splits 
enforcement between agents, where agent $i$ is assigned a share 
$\alpha_{ij}$ of the total constraint demand, yielding the decoupled constraint
\vspace{-0.2em}
\begin{equation}
    2\,\Delta\mathbf{p}_{ij}^{\top} G_i\,\mathbf{u}_i
    \geq 
    -2\,\Delta\mathbf{p}_{ij}^{\top}\boldsymbol{\eta}_i 
    - \alpha_{ij}\,\Upsilon_{ij},
    \label{eq:local_constraint}
\end{equation}
while agent $j$ simultaneously enforces the complementary share 
$(1-\alpha_{ij})$. Together, the two local constraints recover the joint constraint~\eqref{eq:global_constraint} by construction. The capability-aware choice of $\alpha_{ij}$ is 
detailed in Section~\ref{subsec:allocation}.

\subsubsection{Decentralized Safe Control for Individual Agent}

{Given the per-agent safety constraints derived above, agent $i$ computes safe control input by solving the following QP at each control step:}
%
\begin{equation}
\begin{aligned}
    \mathbf{u}_i^* = \arg\min_{\mathbf{u}_i \in \mathbb{R}^2} \quad 
    & \|\mathbf{u}_i - \mathbf{u}_i^{\mathrm{nom}}\|^2 \\
    \mathrm{s.t.} \quad 
    & \mathbf{a}_{ij}^{\top}\,\mathbf{u}_i \geq b_{ij}, 
    \quad \forall\, j \neq i, \\
    & \mathbf u_i\in\mathcal U_i^{\rm acc}(\mathbf x_i)
\end{aligned}
\label{eq:QP}
\end{equation}
where $\mathbf{u}_i^{\mathrm{nom}}$ is the nominal control input, $\mathbf{a}_{ij}^{\top} := 2\,\Delta\mathbf{p}_{ij}^{\top} 
G_i(\mathbf{x}_i) \in \mathbb{R}^{1\times 2}$, and $b_{ij} := 
-2\,\Delta\mathbf{p}_{ij}^{\top}\boldsymbol{\eta}_i(\mathbf{x}_i) - \alpha_{ij}\,\Upsilon_{ij} \in \mathbb{R}$. The kinematic constraint $\mathbf{u}_i \in \mathcal{U}_i^{\mathrm{acc}}(\mathbf{x}_i)$ encodes the model-specific input bounds. The QP~\eqref{eq:QP} has linear constraints with respect to $\mathbf{u}_i$ and can be solved using standard QP solvers (e.g., OSQP~\cite{stellato2020osqp}). 
We assume all agents simultaneously enforce~\eqref{eq:local_constraint} in their respective QPs.%
\phantomsection\label{ass:mutual}

\begin{definition}[Capability-Aware Heterogeneous CBF (CA-HCBF)]
\label{def:cahcbf}
For a heterogeneous multi-robot system, the CA-HCBF for a pair $(i,j)$ is the augmented barrier function $\psi_{ij} := \dot{h}_{ij} + \lambda_1 h_{ij}$, defined over the unified operational space (Section~\ref{subsec:uni_op_space}), with the decentralized constraint~\eqref{eq:local_constraint} enforced by each agent using complementary responsibility parameters $\alpha_{ij}$ and $\alpha_{ji} = 1 - \alpha_{ij}$ detailed in Section~\ref{subsec:allocation}.
\end{definition}


\subsection{Safety Guarantee}
\label{subsec:safety_guarantee}

\begin{theorem}[Heterogeneous Pairwise Safety]
\label{thm:safety}
Suppose all agents simultaneously enforce~\eqref{eq:local_constraint}, each QP~\eqref{eq:QP} is feasible for all $t \geq 0$, and $(\mathbf{x}_i(0), \mathbf{x}_j(0)) \in \mathcal{C}_{ij}^{\mathrm{aug}}$ for all $i \neq j$. Then the augmented safe set $\mathcal{C}_{ij}^{\mathrm{aug}}$ is forward invariant for any pair $(i,j)$ with heterogeneous dynamics satisfying~\eqref{eq:unified}. 
In particular, $h_{ij}(t) \geq 0$ and $\psi_{ij}(t) \geq 0$ for all $t \geq 0$, implying $\|\mathbf{p}_i(t) - \mathbf{p}_j(t)\| \geq R_{ij}$.
\end{theorem}

\begin{proof}
Adding the two local constraints for agents $i$ and $j$, with complementary assignments $\alpha_{ij}$ and $1-\alpha_{ij}$, recovers the coupled constraint~\eqref{eq:global_constraint},
equivalent to $\dot{\psi}_{ij} + \lambda_2\psi_{ij} \geq 0$. 
By the CBF theorem~\cite{ames2019control}, $\psi_{ij}(t) \geq 0$ for all $t \geq 0$ given $(\mathbf{x}_i(0), \mathbf{x}_j(0)) \in \mathcal{C}_{ij}^{\mathrm{aug}}$. Since $\psi_{ij} \geq 0$ implies $h_{ij}(t) \geq h_{ij}(0)e^{-\lambda_1 t} \geq 0$, forward invariance of $\mathcal{C}_{ij}^{\mathrm{aug}}$ follows. 
\end{proof}

{
\begin{remark}
Theorem~\ref{thm:safety} guarantees pairwise safety only while each local safety-constrained QP remains feasible. In dense multi-robot interactions, multiple active pairwise constraints may still render a local QP infeasible even when each pairwise constraint is individually feasible. 
This motivates the proposed capability-aware responsibility allocation {introduced in Sec.~\ref{subsec:allocation} to alleviate decentralized infeasibility.}
\end{remark}
}

While Theorem~\ref{thm:safety} guarantees safety between reference points, 
{collisions between the physical robot bodies must also be prevented}. We therefore define a sufficient pairwise safety distance that guarantees physical collision avoidance.
Let $r_i^{\mathrm{phys}}$ denote the radius of a bounding circle enclosing the physical footprint of robot $i$, and define the per-agent CBF radius $r_i^{\mathrm{cbf}} := x_{r,i} + r_i^{\mathrm{phys}}$, with the pairwise safety distance $R_{ij} := r_i^{\mathrm{cbf}} + r_j^{\mathrm{cbf}}$.

\begin{proposition}[Physical Safety]
\label{prop:physical_safety}
If $\|\mathbf{p}_i(t) - \mathbf{p}_j(t)\| \geq R_{ij}$ for all 
$t \geq 0$, then $\|\mathbf{c}_i(t) - \mathbf{c}_j(t)\| \geq 
r_i^{\mathrm{phys}} + r_j^{\mathrm{phys}}$ for all $t \geq 0$, 
where $\mathbf{c}_i$ denotes the geometric center of robot $i$.
\end{proposition}

\begin{proof}
Since $\|\mathbf{p}_i - \mathbf{c}_i\| = x_{r,i}$, the triangle 
inequality gives $\|\mathbf{c}_i - \mathbf{c}_j\| \geq 
\|\mathbf{p}_i - \mathbf{p}_j\| - x_{r,i} - x_{r,j} \geq 
R_{ij} - x_{r,i} - x_{r,j} = r_i^{\mathrm{phys}} + 
r_j^{\mathrm{phys}}$. 
\end{proof}


\subsection{Capability-Aware Responsibility Allocation}
\label{subsec:allocation}


{
The decomposition in~\eqref{eq:local_constraint} is not unique because any complementary pair satisfying $\alpha_{ij}+\alpha_{ji}=1$ recovers the coupled constraint. 
Decentralized CBF methods commonly adopt equal responsibility
allocation, which implicitly assumes homogeneous maneuverability~\cite{wang2017safety}. 
For heterogeneous kinematic models, however, equal allocation may assign avoidance responsibilities beyond the motion capability of more constrained robots while underutilizing more agile ones. We therefore seek a responsibility allocation that 
(i) reflects each robot's ability to maneuver and 
(ii) avoids assigning responsibilities that exceed what each robot can physically achieve.

}

\subsubsection{Control Capability via Support Function}

{
Since collision avoidance requires each robot to generate acceleration along the pairwise separation direction, we quantify each agent's directional control capability using the support function $\mathcal{S}_{\mathcal U}(\mathbf d):=\max_{\mathbf u\in\mathcal U}\mathbf d^\top\mathbf u$,
which gives the maximum achievable projection of $\mathbf u$ along $\mathbf d$. It is used to evaluate both separating and progress capabilities.


The maximum separating acceleration agent $i$ can produce along the separation direction $\Delta\mathbf{p}_{ij} = \mathbf{p}_i - 
\mathbf{p}_j$ is obtained by evaluating the support function of $G_i\mathcal{U}_i^{\mathrm{acc}}$, the admissible control set transformed into operational space by $G_i$:
\begin{equation}
    \rho_{ij}^{i} := \mathcal{S}_{G_i\mathcal{U}_i^{\mathrm{acc}}}
    \!\left(2\Delta\mathbf{p}_{ij}\right) 
    = \max_{k}\, \mathbf{A}_{ij}\mathbf{u}_i^{(k)},
    \label{eq:support_cap}
\end{equation}
where $\mathbf{A}_{ij} := 2\Delta\mathbf{p}_{ij}^{\top}G_i \in 
\mathbb{R}^{1\times 2}$ is the CBF constraint row vector from~\eqref{eq:local_constraint}, and $\{\mathbf{u}_i^{(k)}\}$ are the vertices of $\mathcal{U}_i^{\mathrm{acc}}$. Incorporating the heading-rotation drift, effective separating capability is \vspace{-0.6em}
\begin{equation}
    \bar{\rho}_{ij}^{i} := \max\!\left(\rho_{ij}^{i} 
    + 2\,\Delta\mathbf{p}_{ij}^{\top}\boldsymbol{\eta}_i,\ 0\right),
    \label{eq:eff_cap}
\end{equation}
and symmetrically for agent $j$. Since $\boldsymbol{\eta}_i$ 
contributes to the CBF constraint independently of $\mathbf{u}_i$, $\bar{\rho}_{ij}^{i}$ captures the total effective separating capability of agent $i$, combining both its control authority and the drift already acting in the avoidance direction.

\subsubsection{Progress-Oriented Allocation}
{{Using the support-function-based directional capability measure, we next determine pairwise avoidance responsibilities among the feasible allocations. Specifically,}
we assign more responsibility to robots that can accommodate avoidance with the smallest deviation from their nominal motion, allowing robots that can maneuver more easily to take on more avoidance effort while reducing the burden on more constrained robots.

}
Let $\mathbf{d}_i := \dot{\mathbf{p}}_i^{\mathrm{nom}} - \dot{\mathbf{p}}_i$ denote the velocity gap between the nominal and actual operational-space velocity. The progress capability $\sigma_i := \max_{\mathbf{u}_i \in \mathcal{U}_i^{\mathrm{acc}}} \mathbf{d}_i^{\top} G_i\,\mathbf{u}_i$ measures the maximum acceleration agent $i$ can accelerate toward its nominal motion direction, giving the progress-oriented responsibility parameter
\begin{equation}
    \alpha_{ij}^{\mathrm{prog}} := \frac{\sigma_i}{\sigma_i + \sigma_j 
    + \epsilon},
    \label{eq:alpha_prog}
\end{equation}
where $\epsilon > 0$ prevents division by zero. 
A larger $\sigma_i$ indicates greater authority to restore nominal progress after performing avoidance maneuvers.

\subsubsection{Feasibility-Aware Clipping}
{Although the progress-oriented allocation promotes goal-reaching efficiency, it does not guarantee that the assigned responsibility {always} remains physically achievable. We therefore derive the feasible interval of responsibility allocations based on each robot's achievable control capability and clip the resulting allocation to this interval whenever necessary.
} 
Agent $i$'s pairwise constraint for neighbor $j$ is individually feasible over $\mathcal{U}_i^{\mathrm{acc}}$ iff $\bar{\rho}_{ij}^{i} \geq -\alpha_{ij}\Upsilon_{ij}$, and symmetrically for agent $j$. This condition becomes active when $\Upsilon_{ij} < 0$, i.e., when the agents are approaching each other and the CBF constraint demands a separating response. Dividing by $-\Upsilon_{ij} > 0$ yields the feasible interval for~$\alpha_{ij}$: \vspace{-0.5em}
\begin{equation}
    \alpha_{\min} := \max\!\left(0,\ 1 - 
    \frac{\bar{\rho}_{ij}^{j}}{-\Upsilon_{ij}}\right), \;
    \alpha_{\max} := \min\!\left(1,\ 
    \frac{\bar{\rho}_{ij}^{i}}{-\Upsilon_{ij}}\right),
    \label{eq:alpha_interval}
\end{equation}

where $\alpha_{\max}$ is the maximum share agent $i$ can accommodate and $\alpha_{\min}$ is the minimum to keep agent $j$'s QP feasible. For $\Upsilon_{ij}\ge 0$, we directly use
$\alpha_{ij}=\alpha_{ij}^{\rm prog}$. The final parameter is
\vspace{-0.2em}
\begin{equation}
    \alpha_{ij} = 
    \begin{cases}
        \mathrm{clip}\!\left(\alpha_{ij}^{\mathrm{prog}},\ 
        \alpha_{\min},\ \alpha_{\max}\right) 
        & \text{if } \alpha_{\min} \leq \alpha_{\max}, \\
        \tfrac{1}{2}
        & \text{otherwise},
    \end{cases}
    \label{eq:alpha_final}
\end{equation}
{for $\Upsilon_{ij}<0$, the interval 
$[\alpha_{\min},\alpha_{\max}]$ characterizes the range of pairwise responsibility assignments compatible with the directional control
capabilities of both agents.
When this interval is nonempty, Eq.~\eqref{eq:alpha_final} minimally modifies the progress-oriented allocation by clipping $\alpha_{ij}^{\mathrm{prog}}$ to this feasible range. When $\alpha_{\min}>\alpha_{\max}$, the combined capability $\bar{\rho}_{ij}^{i}+\bar{\rho}_{ij}^{j}$ is insufficient to cover the pairwise demand $-\Upsilon_{ij}$, so no decentralized split can render both individual pairwise constraints feasible simultaneously.
}

{

{Feasibility-aware clipping eliminates responsibility assignments that are individually unrealizable for either robot in a pair, thereby reducing one important source of decentralized QP infeasibility. However, the resulting local QP may still become infeasible when multiple simultaneously active pairwise constraints impose collectively incompatible demands on the same robot. In such cases, we apply maximum deceleration as a conservative fallback to reduce the relative closing speed; more sophisticated recovery strategies, such as constraint relaxation or hierarchical prioritization~\cite{ames2016control_qp,lee2023hierarchical,chen2020guaranteed}, could also be incorporated.}
}


\section{Experiments}
\label{sec:experiments}

\begin{table*}[!t]
\centering
\caption{Performance comparison across team sizes ($N \in \{10, 20, 30\}$, 50 trials each).}
\label{tab:main_results}
\setlength{\tabcolsep}{4pt}
\renewcommand{\arraystretch}{1.1}
\begin{tabular}{l ccc ccc ccc}
\toprule
& \multicolumn{3}{c}{$N=10$}
& \multicolumn{3}{c}{$N=20$}
& \multicolumn{3}{c}{$N=30$} \\
\cmidrule(lr){2-4}\cmidrule(lr){5-7}\cmidrule(lr){8-10}
\textbf{Method}
& \texttt{AR}\,(\%)\,$\uparrow$ & \texttt{\# Viol}\,$\downarrow$ & \texttt{Mean\,V}\,$\downarrow$
& \texttt{AR}\,(\%)\,$\uparrow$ & \texttt{\# Viol}\,$\downarrow$ & \texttt{Mean\,V}\,$\downarrow$
& \texttt{AR}\,(\%)\,$\uparrow$ & \texttt{\# Viol}\,$\downarrow$ & \texttt{Mean\,V}\,$\downarrow$ \\
\midrule
\text{CA-HCBF (Ours)}
  & \cellcolor{gray!25}\textbf{97.4} & \cellcolor{gray!10}{0.1 (0.5)}  & \cellcolor{gray!10}{0.001 (0.008)}
  & \cellcolor{gray!25}\textbf{94.7} & \cellcolor{gray!10}{1.2 (2.0)}  & \cellcolor{gray!25}\textbf{0.003 (0.008)}
  & \cellcolor{gray!25}\textbf{89.6} & {5.1 (9.6)}  & \cellcolor{gray!25}\textbf{0.018 (0.035)} \\
\text{APF + Tracking} ($w{=}0.9$)
  & \cellcolor{gray!10}{94.2} & 4.3 (2.8)   & 0.273 (0.112)
  & \cellcolor{gray!10}{90.9} & 22.7 (6.4)  & 0.280 (0.075)
  & \cellcolor{gray!10}{87.1} & 59.7 (11.6) & 0.279 (0.066) \\
\text{APF + Tracking($w{=}0.5$)}
  & 89.4 & 2.1 (2.0)  & 0.200 (0.147)
  & 86.1 & 13.5 (3.8) & 0.267 (0.064)
  & 80.7 & 37.9 (9.4) & 0.271 (0.048) \\
\text{APF + Tracking($w{=}0.1$)}
  & 46.2 & \cellcolor{gray!25}\textbf{0.0 (0.1)} & 0.005 (0.034)
  & 22.5 & \cellcolor{gray!25}\textbf{0.1 (0.4)} & 0.041 (0.124)
  & 15.5 & \cellcolor{gray!25}\textbf{1.0 (1.3)} & 0.179 (0.206) \\
\text{APF + HOCBF (Hom.)}
  & 42.4 & 2.5 (1.6)   & 0.196 (0.124)
  & 36.6 & 11.5 (7.3)  & 0.271 (0.086)
  & 29.9 & 25.2 (9.9) & 0.264 (0.065) \\
\text{sRCBF (APF)}
  & 33.4 & \cellcolor{gray!25}\textbf{0.0 (0.1)} & \cellcolor{gray!25}\textbf{0.000 (0.000)}
  &  3.9 & 1.4 (3.1) & 0.056 (0.097)
  &  1.0 & \cellcolor{gray!10}{1.4 (2.8)} & 0.046 (0.078) \\
\text{sRCBF (PD)}
  & 24.6 & 1.1 (2.2)  & 0.028 (0.100)
  &  4.4 & 5.7 (10.5) & \cellcolor{gray!10}{0.029 (0.071)}
  &  0.9 & 3.2 (3.8)  & \cellcolor{gray!10}{0.040 (0.068)} \\
\bottomrule
\end{tabular}
\begin{minipage}{\textwidth}
\small
\vspace{2pt}
Results are reported as mean (std). \textbf{Bold} with \colorbox{gray!25}{dark shading} indicates best performance per metric; \colorbox{gray!10}{light shading} indicates second best. \texttt{AR}\,(\%): goal arrival rate; \texttt{\# Viol}: number of 
pairwise safety violations; \texttt{Mean\,V}: mean violation depth (m). 
\vspace{-1.5em}
\end{minipage}
\end{table*}

{
In this section, we aim to 1) evaluate the navigation performance of CA-HCBF and 2) analyze the effect of the proposed capability-aware responsibility allocation on practical decentralized feasibility.
}
{
Specifically, we validate CA-HCBF through
(i) random start-goal navigation simulations with up to $N=30$ heterogeneous robots,
(ii) an ablation study of the proposed responsibility allocation strategy, and
(iii) a five-robot physical experiment involving all five kinematic classes. 
}

\subsection{Experimental Setup}
\label{subsec:setup}
\begin{figure}[t]
    \captionsetup{skip=2pt}
    \centering
    \includegraphics[width=0.48\textwidth]{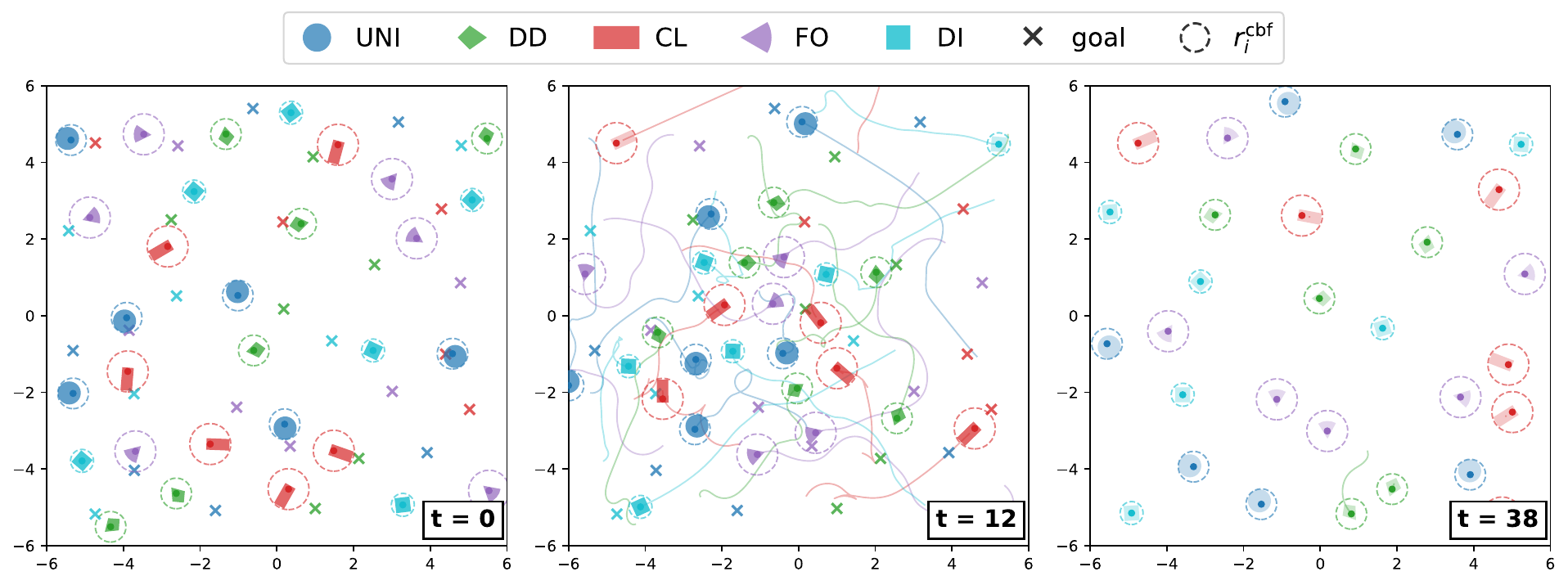}
    \caption{Random start--goal navigation with 30 heterogeneous robots
(6 per model) at three timestamps. Starting from random initial positions,
all robots reach their assigned goals by $t=38$; \texttimes~markers indicate
the goals, and trailing curves show the executed trajectories. Robot classes
are represented by shape and color: UNI (blue circle), DI (light-blue square),
DD (green diamond), CL (red rectangle), and FO (purple sector). Dashed circles
denote the CBF safety radii $r_i^{\mathrm{cbf}}$ centered at the reference
points (Section~\ref{subsec:safety_guarantee}).
    }
    \label{fig:random_start_goal}
    \vspace{-1.5em}
\end{figure}

\subsubsection*{\textbf{Simulation Environment}}

{
We evaluate CA-HCBF on a random start-goal navigation benchmark in a 2D plane. Initial and goal positions are sampled from an
$11\times11\,\mathrm{m}^2$ region with minimum start--goal and
initial inter-agent distances of $4.0$\,m and $1.5$\,m, respectively
(Fig.~\ref{fig:random_start_goal}). We consider teams of
$N\in\{10,20,30\}$ robots, evenly distributed across the five
kinematic classes, over 50 random trials per team size.
Each simulation runs for at most 1000 steps with
$\Delta t=0.05$\,s.\footnote{A robot is considered to have reached its
goal when it is within $0.3$\,m of the assigned goal.}
The CBF gains are $\lambda_1=\lambda_2=2.0$, and all QPs are solved
using OSQP~\cite{stellato2020osqp} with a solver tolerance of $10^{-5}$.
}

All robot classes use the same speed limit, $v_{\max}=1.0$~m/s.
The reference-point offset is $0.1$~m for UNI/DD and $0.2$~m for
CL/FO; CL/FO additionally use a wheelbase of $\ell=0.5$~m.
DI, UNI, and DD are modeled with circular footprints of radius
$0.3$~m, whereas CL/FO use $0.3\times0.6$~m rectangular footprints,
with $r_i^{\mathrm{phys}}$ given by the radius of the corresponding
minimum enclosing circle.

All methods use the same decentralized artificial potential field
(APF) controller~\cite{khatib1986apf} for nominal navigation.
The desired operational-space velocity is mapped to a nominal
acceleration using proportional tracking. For DI robots, the resulting Cartesian acceleration is used directly
as the nominal input, whereas the nonholonomic classes obtain their
nominal inputs through forward-speed and heading tracking.\footnote{Unless otherwise stated, we use $w=0.9$, with
$\dot{\mathbf p}_i^{\rm nom}
=w\mathbf f_i^{\rm att}+(1-w)\mathbf f_i^{\rm rep}$ and a repulsive
field whose magnitude scales as $1/d_{ij}^3$.
}

\subsubsection*{\textbf{Evaluation Metrics}}
{
Navigation performance and safety are evaluated using three metrics. \texttt{Arrival rate} (\texttt{AR}, \%) is the fraction of robots reaching their goals within $T_{\max}=1000$ simulation steps.
\texttt{Violation count} (\texttt{\# Viol}) is the number of distinct pairwise
safety-violation events for which
$\|\mathbf p_i-\mathbf p_j\|<R_{ij}$.
\texttt{Mean violation depth} (\texttt{Mean V}) is the average penetration depth $\max(0,R_{ij}-\|\mathbf p_i-\mathbf p_j\|)$ over all violating pairs and timesteps, measuring the severity of pairwise safety violations.
For the ablation study, we additionally report the \texttt{local infeasibility rate}, defined as the percentage of 
{control timesteps at which a robot's}
local safety-constrained QP is infeasible.
}

\vspace{-0.5em}

\subsection{{Baselines}}
\label{subsubsec:baselines}

We compare against three representative baselines covering complementary approaches to decentralized multi-robot safety. \textbf{(1) APF + Tracking ($w\in\{0.1,0.5,0.9\}$)}
uses the nominal APF controller~\cite{khatib1986apf} without a safety filter. Varying the goal-seeking weight $w$ exposes the
safety--efficiency trade-off achievable through nominal reactive navigation alone.
\textbf{(2) APF + HOCBF (Homogeneous)}
uses the same second-order HOCBF construction
~\cite{xiao2019hocbf,taylor2022safe_backstep}, but models all robots as
holonomic double integrators. Accordingly, the heading-induced drift
$\boldsymbol{\eta}_i$ in~\eqref{eq:op_second} is omitted, while
nonholonomic robots track the resulting operational-space commands
through speed and heading controllers.
\textbf{(3) sRCBF (PD / APF)} applies the smooth Robust CBF (sRCBF)~\cite{kim2025srcbf} to each robot's canonical model and accounts for neighboring robots through worst-case velocity bounds. We evaluate both the original proportional-derivative (PD) nominal controller from~\cite{kim2025srcbf}
and the common APF nominal controller.
Together, these baselines compare CA-HCBF against nominal collision
avoidance, homogeneous-model HOCBF enforcement, and worst-case
robust-CBF modeling. The effect of capability-aware responsibility allocation is examined separately in the ablation study.


\subsection{Random Start-Goal Navigation}
\label{subsec:navigation}

{The main results are summarized in Table~\ref{tab:main_results}. CA-HCBF achieves the highest arrival rate at $N=10$ and $N=20$ ($97.4\%$ and $94.7\%$) and maintains the highest arrival rate at $N=30$ (89.6\%), while maintaining consistently low safety violations. Even at $N=30$, the mean violation depth remains only $0.018$\,m, indicating that CA-HCBF substantially reduces both the frequency and severity of safety violations.
} \vspace{-1em}

{
In contrast, the APF-only baselines expose a fundamental safety--efficiency trade-off governed by $w$: increasing $w$ improves arrival rates but leads to substantially more safety violations, whereas decreasing $w$ suppresses violations at the cost of poor task completion. These results indicate that the balance between navigation efficiency and safety obtained by CA-HCBF cannot be achieved by APF weight tuning alone.
} \vspace{-1em}

{

APF + HOCBF achieves arrival rates of only $42.4\%$, $36.6\%$, and $29.9\%$ at $N=10$, $20$, and $30$, respectively, while exhibiting substantially more frequent and severe safety violations than CA-HCBF. These results highlight the importance of explicitly modeling heterogeneous operational-space dynamics, including the heading-rotation drift term $\boldsymbol{\eta}_i$ for nonholonomic robots, when constructing decentralized CBF constraints.

}

{The sRCBF baselines exhibit the opposite behavior: despite achieving low violation counts, both exhibit near-zero arrival rates at $N=20$ and $N=30$ (below $5\%$). 
By treating neighboring robots as worst-case dynamic obstacles, they become increasingly conservative as team size grows, often freezing agents before reaching their goals.
} \vspace{-0.5em}

\subsection{Ablation Study}
\label{subsec:ablation}
\vspace{-0.em}

\begin{figure}[!t]
    \captionsetup{skip=0pt}
    \centering
    \includegraphics[width=0.35\textwidth]{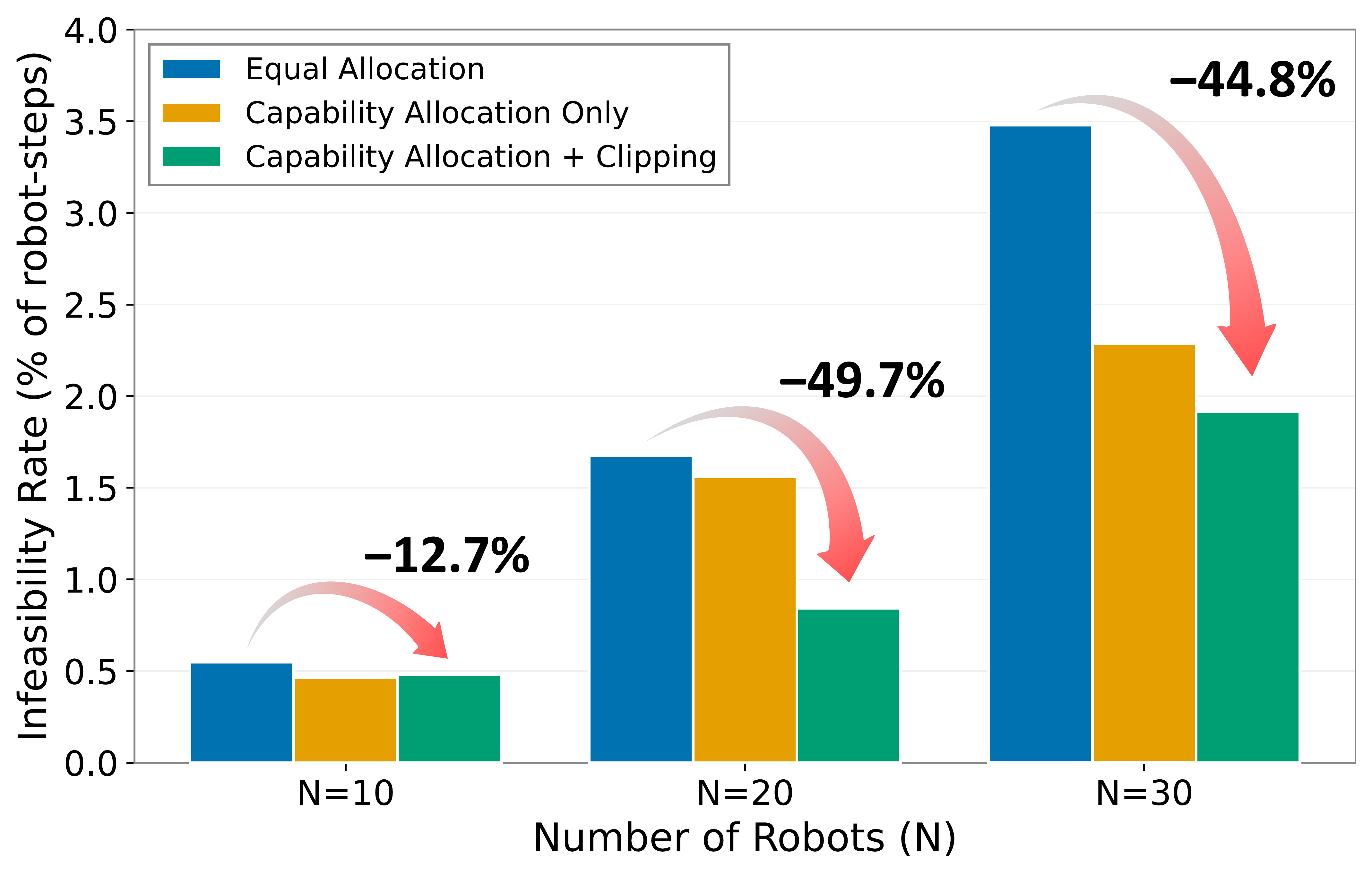}
    \caption{
    Local infeasibility rates across different team sizes ($N\in\{10,20,30\}$). Lower values indicate improved decentralized feasibility.
    }
    \label{fig:qp_failures}
    \vspace{-1.5em}
\end{figure}

{
To isolate the effect of the proposed responsibility allocation on decentralized feasibility, we compare three allocation strategies:
(1) \texttt{Equal} allocation ($\alpha_{ij}=\tfrac{1}{2}$), (2) \texttt{Cap.}, capability-aware allocation without feasibility-aware clipping, and (3) \texttt{Cap.+Clip}, the full capability- and feasibility-aware allocation. The comparison is performed on the same random start-goal navigation benchmark described in Section~\ref{subsec:setup}, with all experimental settings held fixed except for the responsibility allocation strategy. The \emph{local infeasibility rate} is defined as the percentage of agent-timesteps during which the local safety-constrained QP is infeasible.

}

{As shown in Fig.~\ref{fig:qp_failures}, the local infeasibility rate increases with team size for all allocation strategies, reflecting the increasing difficulty of maintaining decentralized feasibility under dense multi-robot interactions. 
While the improvement is modest for sparse teams ($N=10$), it becomes more pronounced as interaction density increases. At $N=30$, the local infeasibility rate decreases from $3.48\%$ under \texttt{Equal} allocation to $2.29\%$ under \texttt{Cap.}, corresponding to a $34.2\%$ reduction. Incorporating feasibility-aware clipping further reduces the rate to $0.84\%$ at $N=20$ and $1.92\%$ at $N=30$, corresponding to overall reductions of $49.7\%$ and $44.8\%$, respectively. These results show that capability-aware allocation improves decentralized feasibility, while feasibility-aware clipping provides additional gains.


\begin{table}[t]
\centering
\caption{Per-agent local infeasibility rate (\%) by kinematic model under different responsibility allocation strategies.}
\label{tab:permodel_infeasible}
\setlength{\tabcolsep}{3pt}
\renewcommand{\arraystretch}{1.1}
\begin{tabular}{l ccc ccc}
\toprule
& \multicolumn{3}{c}{$N=20$}
& \multicolumn{3}{c}{$N=30$} \\
\cmidrule(lr){2-4}\cmidrule(lr){5-7}
\textbf{Model}
& \texttt{Equal}
& \texttt{Cap.}
& \texttt{Cap.+Clip}
& \texttt{Equal}
& \texttt{Cap.}
& \texttt{Cap.+Clip} \\
\midrule
DI
    & 0.12 & 0.12 & 0.12
    & 0.21 & 0.26 & 0.23 \\
UNI
    & 0.27 & 0.22 & 0.24
    & 0.39 & 0.43 & 0.38 \\
DD
    & 0.14 & 0.22 & 0.23
    & 0.42 & 0.37 & 0.48 \\
CL
    & 0.68 & 1.01 & 0.58
    & 2.89 & 1.52 & 1.86 \\
FO
    & 7.17 & 6.22 & 3.04
    & 13.47 & 8.86 & 6.62 \\

\bottomrule
\end{tabular}

\begin{minipage}{\columnwidth}
\small
\vspace{2pt}
Entries denote the per-agent local infeasibility rate (\%). Lower values indicate better decentralized feasibility.
\end{minipage}
\vspace{-2em}
\end{table}

Table~\ref{tab:permodel_infeasible} shows that the improvement is concentrated on the more kinematically constrained models, particularly FO and CL. At $N=30$, the infeasibility rate of FO decreases from $13.47\%$ under \texttt{Equal} allocation to $6.62\%$ with \texttt{Cap.+Clip}, while CL decreases from $2.89\%$ to $1.86\%$. In contrast, DI, UNI, and DD exhibit only minor changes across allocation strategies. These results indicate that the overall reduction in local infeasibility is primarily driven by improvements for the most kinematically constrained robots.

} \vspace{-3pt}

\subsection{Physical Robot Validation}
\label{subsec:physical}
\vspace{-3pt}

{
To demonstrate the proposed capability-aware coordination on physical
robots, we conduct experiments with five LIMO Pro robots~\cite{limopro},
configured to realize the five kinematic models considered in this work
(DI, UNI, DD, CL, and FO). We use $v_{\max}=0.3\,\mathrm{m/s}$ and a body size of $322\times222\,\mathrm{mm}$. As shown in Fig.~\ref{fig:dialogue}, the robots are initialized symmetrically with goals on opposite sides of the workspace, creating heterogeneous interactions near the central
region. The observed trajectories are consistent with the proposed
capability-aware allocation: more maneuverable robots execute larger
avoidance maneuvers around more constrained robots.
\footnote{Since the LIMO Pro accepts velocity commands,
we convert the acceleration input to
$\boldsymbol{\nu}_i^{k+1}
=\boldsymbol{\nu}_i^{k}+\mathbf{u}_i^{k}\Delta t$
before low-level tracking.}
} \vspace{-0.5em}

\section{Conclusion}


{
This paper presents CA-HCBF, a decentralized CBF framework for heterogeneous multi-robot systems with distinct kinematic constraints. By introducing a unified second-order operational-space formulation and capability-aware responsibility allocation, CA-HCBF enables pairwise safety constraint construction across heterogeneous kinematic models while allocating avoidance responsibility according to each robot's directional motion capability. Feasibility-aware clipping further reduces local QP infeasibility by preventing individually unrealizable responsibility assignments beyond an individual robot's control capability. 
{Simulation studies with up to 30 heterogeneous robots and hardware
experiments demonstrate high task completion, low safety violations,
and improved decentralized feasibility through the proposed
capability-aware responsibility allocation framework.}

Future work will extend the proposed pairwise responsibility allocation to neighbor-level coordination 
{to explicitly address dense interactions and deadlock,}
 while incorporating perception and model uncertainty into the capability-aware framework to improve robustness in real-world deployments.} \vspace{-0.7em}



%
\bibliographystyle{IEEEtran} \footnotesize
\bibliography{reference_clean_backup}

\end{document}